\documentclass[conference]{IEEEtran}
\IEEEoverridecommandlockouts
\usepackage{cite}
\usepackage{amsmath,amssymb,amsfonts}
\usepackage{algorithm}
\usepackage{algorithmic}
\usepackage{graphicx}
\usepackage{textcomp}
\usepackage{xcolor}

\def\BibTeX{{\rm B\kern-.05em{\sc i\kern-.025em b}\kern-.08em
    T\kern-.1667em\lower.7ex\hbox{E}\kern-.125emX}}

\newcommand{\E}{\ensuremath{{\mathbb{E}}}}
\newcommand{\R}{\ensuremath{{\mathbb{R}}}}

\newtheorem{assumption}{Assumption}
\newtheorem{theorem}{Theorem}
\newtheorem{lemma}{Lemma}
\newtheorem{remark}{Remark}

\begin{document}

\title{Expectation Maximization (EM) Converges for General Agnostic Mixtures
}

\author{\IEEEauthorblockN{Avishek Ghosh}
\IEEEauthorblockA{\textit{Department of Computer Science and Engineering} \\
\textit{IIT Bombay}\\
Mumbai, India \\
email: avishek\_ghosh@iitb.ac.in}
}

\maketitle

\begin{abstract}
Mixture of linear regression is well studied in statistics and machine learning, where the data points are generated probabilistically using $k$ linear models. Algorithms like Expectation Maximization (EM) may be used to recover the ground truth regressors for this problem. Recently, in \cite{pal2022learning,ghosh_agnostic} the mixed linear regression problem is studied in the agnostic setting, where no generative model on data is assumed. Rather, given a set of data points, the objective is \emph{fit} $k$ lines by minimizing a suitable loss function. It is shown that a modification of EM, namely gradient EM converges exponentially to appropriately defined loss minimizer even in the agnostic setting.

In this paper, we study the problem of \emph{fitting} $k$ parametric functions to given set of data points. We adhere to the agnostic setup. However, instead of fitting lines equipped with quadratic loss, we consider any arbitrary parametric function fitting equipped with a strongly convex and smooth loss. This framework encompasses a large class of problems including mixed linear regression (regularized), mixed linear classifiers (mixed logistic regression, mixed Support Vector Machines) and mixed generalized linear regression. We propose and analyze gradient EM for this problem and show that with proper initialization and separation condition, the iterates of gradient EM converge exponentially to appropriately defined population loss minimizers with high probability. This shows the effectiveness of EM type algorithm which converges to \emph{optimal} solution in the non-generative setup beyond mixture of linear regression.
\end{abstract}


\section{Introduction}
\label{sec:intro}
Suppose we have $n$ data points $\{x_i,y_i\}_{i=1}^n$, where $x_i \in \R^d$ and $y_i \in \R$ independently sampled from a (joint) distribution $\mathcal{D}$. Our objective is to learn a list of $k$ functions $y = f_{\theta_j}(x)$ (parameterized by $\theta_j$) where $j\in [k]$\footnote{For a positive integer $m$, we use $[m]$ to denote the set $\{1,2,\ldots,m\}$.}. This problem is well studied in the linear setting where $f_{\theta_j}(x) = \langle x,\theta_j\rangle$ and the labels are generated by the ground truth parameters, $\Tilde{\theta}_1,\ldots,\Tilde{\theta}_k$ in the following way:
\begin{align*}
x_i \sim \mathcal{N}(0,I_d), \theta \sim \mathrm{Unif}\{\tilde{\theta_1}, \dots, \tilde{\theta_k}\},  y_i | \theta \sim \mathcal{N}(x^T\theta, \sigma^2),
\end{align*}
for $i \in [n]$. This is known as mixed linear regression (\cite{yi2014alternating, klusowski2019estimating,chaganty2013spectral}). Algorithmic convergence, rates, bounds on sample complexity in terms of $d,\sigma^2,n$ as well as prediction error estimates of $\Tilde{\theta}_1,\ldots,\Tilde{\theta}_k$ are well-known in the literature (\cite{yi2016solving,balakrishnan2017statistical,kwon2018global}).

In \cite{pal2022learning,ghosh_agnostic}, the authors consider the framework of \emph{agnostic} learning for mixed linear regression. In this framework, no generative model is assumed on $y$ and hence no ground truth parameters. The problem is viewed from a \emph{fitting} viewpoint, where the learner is given data points $\{x_i,y_i\}_{i=1}^n$, and the goal is to find $k$ optimal mappings by minimizing certain loss functions. This problem is studied in  \cite{pal2022learning,ghosh_agnostic} for the linear setting where the authors fit $k$ lines.

In this work, we consider an agnostic and general learning theoretic setup, but for problems beyond mixed linear regression. We consider the same setup of \cite{pal2022learning,ghosh_agnostic} where we do not assume any generative model for $y$ and instead of fitting $k$ lines, we let the \emph{fitting function} $f_\theta(.)$ to be any arbitrary mapping function parameterized by $\theta$.

Since in our setup, there are no ground truth parameters, we need to define the \emph{target} parameters through suitable loss functions. We denote a loss function $\ell: \R^{d \times k}\to \R$ for sample $(x,y)$ as $\ell(\theta_1, \theta_2, \dots, \theta_k; x,y)$. We define the population (or average) loss as
\begin{align*}
    \mathcal{L}(\theta_1, \theta_2, \dots, \theta_k) \equiv \E_{(x,y)\sim \mathcal{D}}\ell(\theta_1, \theta_2, \dots, \theta_k; x,y).
\end{align*}
Furthermore, we define the minimizers of the loss function as
\begin{align*}
    (\theta^\ast_1, \dots, \theta^\ast_k) \equiv \arg \min_{\theta_1,\ldots,\theta_k}  \mathcal{L}(\theta_1, \theta_2, \dots, \theta_k).
\end{align*}
Note that learning $k$ parameters make sense when we are allowed to output a list of $k$ predictions $[f_{\theta^*_j}(x)]_{j=1}^k$. In \cite{pal2022learning,ghosh_agnostic}, the authors dub a list of predictions \emph{good} is if at least one of the labels in he list is a good predictor. 

\vspace{2.5mm}
In this paper, we consider the following loss function:
\begin{align}\label{eq:softmin}
\ell(\theta_1, \dots, \theta_k; x,y) = \sum_{j=1}^k p_{\theta_1,..,\theta_k}(x,y;\theta_j) F(x,y;\theta_j)
\end{align}
\begin{align*}	\text{where} \quad p_{\theta_1,..,\theta_k}(x,y;\theta_j) = \frac{e^{-\beta F(x,y;\theta_j)}}{\sum_{l=1}^k e^{-\beta F(x,y;\theta_l)}}.
\end{align*}
The loss function is called a \emph{soft-min} loss and is typically used for analyzing mixtures (\cite{kwon2018global,balakrishnan2017statistical,daskalakis_ten,daskalakis2014faster}) since it is related (a lower bound on) to the \emph{optimal} maximum likelihood loss in the generative setup. Here $F(x,y;\theta_j)$ denotes the \emph{base} loss function on data point $(x,y)$ for the $j$-th parameter $\theta_j$ and $\beta \ge 0$ is the inverse temperature parameter. Observe that when $\beta \to \infty$,  Equation~\ref{eq:softmin} corresponds to the min-loss defined as 
\begin{align*}
    \ell(\theta_1,\ldots,\theta_k) = \min_{j \in [k]} F(x,y;\theta_j).
\end{align*}
On the other hand, if $\beta =0,$ Equation~\ref{eq:softmin} reduces to the average of the base errors, if the label is uniformly chosen.

For the problem of mixed linear regression, we take the fitting function $f_{\theta_j} = \langle x,\theta_j \rangle$ and a  standard choice for base loss is $F(x,y;\theta_j) = \left[ y - \langle x,\theta_j \rangle \right]^2$, (quadratic loss) \cite{balakrishnan2017statistical,kwon2018global}. Even in the agnostic setup, \cite{pal2022learning,ghosh_agnostic} consider the same quadratic loss.

In this work, we consider general (arbitrary) fitting function $f_{\theta_j}$ (beyond linear) and \emph{general} base losses $F(x,y;\theta_j)$ (beyond quadratic). In particular, we assume the base losses $F(x,y;\theta_j)$ to be $M$-smooth and $m$-strongly convex with respect to $\theta$. We define this more formally in the sequel. The class of smooth and strongly convex losses is much richer and includes the quadratic loss (with regularization). Hence, the problem we study is richer than  the mixed linear regression problem in the agnostic setting.

We mention that strongly convex and smooth losses are quite common in machine learning and statistics. Ranging from classical optimization theory (see \cite{Wright_Recht_2022}), distributed or Federated optimization (\cite{scaffold}), robust learning (\cite{dong_byzantine,ghosh2021communication}), meta learning (\cite{fallah2020personalized}) to modern theory of over-parameterized interpolation (\cite{ma2018power}), such assumptions are ubiquitous.

Of course, in practice, we de not have access to the \emph{population} loss function $\mathcal{L}(\theta_1, \theta_2, \dots, \theta_k)$. Instead, we work with the \emph{empirical} loss function denoted by 
\begin{align*}
    L(\theta_1,\ldots,\theta_k) = \frac{1}{n}\sum_{i=1}^n \ell (\theta_1,\ldots,\theta_k;x_i,y_i).
\end{align*}

Solving the above empirical loss may not be straightforward in many cases. In \cite{yi2014alternating}, it is showed that even for mixed linear regression (with squared base loss), the minimization problem in NP hard\footnote{The authors of \cite{yi2014alternating} use a reduction to the subset-sum problem which is known to be NP-hard.}.

Usually for mixture problems, a popular choice is to use an iterative algorithm namely Expectation Maximization (EM). EM starts with some initial estimate of the parameters and based on the observed data, it refines them by taking the expectation and the maximization step repeatedly. It is shown in \cite{balakrishnan2017statistical} and \cite{kwon2018global} that for mixture of linear regressions, a suitable initialized EM algorithm converges exponentially fast to the ground truth parameters with high probability. Later \cite{daskalakis_ten} showed that the convergence speed of EM for such problems is \emph{very fast}. All the above-mentioned papers assume a generative setup on data and minimize the empirical \emph{soft-min} loss with squared base loss.

Another variation of EM algorithm popularly used is known as gradient EM (\cite{balakrishnan2017statistical,zhu2017high,wang2020differentially,pal2022learning,ghosh_agnostic}). In this setup instead of the maximization step, we take a gradient step with a suitably chosen step size. Note that a gradient step instead of a full-blown optimization is amenable to analysis and hence in this paper also, we focus on gradient EM algorithm.

Recently, \cite{pal2022learning,ghosh_agnostic} consider the mixed linear regression problem without generative assumption on data (agnostic). They work with \emph{soft-min} (as well as \emph{min} loss) with quadratic function as base loss and use the iterative algorithm gradient EM. It is shown in \cite{pal2022learning} that if the initial estimates are within $\mathcal{O}(1/\sqrt{d})$ of the population loss minimizers $\{\theta^*_j\}_{j=1}^k$, then grad EM  converges exponentially fast. Later on, \cite{ghosh_agnostic} use the same framework and improve the initialization requirement from $\mathcal{O}(1/\sqrt{d})$ to $\Theta(1)$ while maintaining the same (exponential) rate of convergence.

Moreover, in both \cite{pal2022learning,ghosh_agnostic}, it is assumed that data (covariates) are sampled from a standard Gaussian distribution, i.e., $x_i \stackrel{i.i.d}{\sim} \mathcal{N}(0,I_d) $. While such assumptions are standard and featured in past works such as\cite{netrapalli2015phase,yi2016solving}, it is not desirable in many cases specially with data having heavy tails.

We consider the agnostic learning of general mixtures. We do not make any generative assumptions on data. Moreover, instead of fitting $k$-lines like \cite{pal2022learning,ghosh_agnostic}, we fit $k$ functions $f_{\theta_j}(.)$ parameterized by $\theta_1,\ldots,\theta_k$. If $f_{\theta_j}(x)=\langle x, \theta_j \rangle$ we get back the agnostic mixed linear regression problem. Moreover,  \cite{pal2022learning,ghosh_agnostic} consider only quadratic loss  as the \emph{base} loss function $F()$. We do not put any restriction on the fitting function $f_{\theta_j}(.)$. The only requirement is that the \emph{base} loss $F(.)$ is strongly convex and smooth. As a result, our current framework encompasses a wide variety of problems beyond the (agnostic) mixed linear regression. 

We now collect some examples (problem instances) where the mapping function $f_{\theta_j}$ may be non-linear and the base function $F(.)$ is strongly convex and smooth.

\vspace{2.5mm}

\emph{(a) Ridge Regularized Mixed Linear Regression:} This is the classical setup with $F(x,y,\theta) = [y - \langle x,\theta \rangle]^2 + \lambda \|\theta\|^2$ as (regularized) quadratic loss. Note that $F(x,y,\theta)$ is smooth and strongly convex for $\lambda > 0$.

\emph{(b) Mixed Logistic Regression:} This loss function is particularly interesting when we study a mixture of classifiers (see \cite{mixture_classifiers}). We use the log-loss as base loss given by $F(x,y;\theta) = \log (1+\exp(-yx^T \theta)$. Note that $F(x,y;\theta)$ is smooth but not strongly convex in general. However, if we restrict $x \in \mathcal{X} \subseteq \mathbb{R}^d$ with finite diameter,\footnote{For a set $\mathcal{X}$, the diameter is denoted by $\max_{x_1,x_2 \in \mathcal{X}}\|x_1-x_2\|$.} we obtain strong convexity. Moreover, adding $\ell_2$ regularization also makes the objective smooth and strobngly convex.


\emph{(c) Mixtures of Support Vector Machines:} This is also an example of mixture of classifiers (margin-based). Here we use a (squared) Hinge loss given by $F(x,y,\theta) = \max(0,1-yx^T \theta) + \frac{\lambda}{2}\|\theta\|^2$. Note that this is strongly convex and smooth for $\lambda>0$.

\emph{(d) Mixtures of Generalized Linear Models:} This is an extension to the mixture of linear regression with $F(x,y,\theta) = [y - g(x^T\theta)]^2 + \lambda \|\theta\|^2$ where $g$ is the link function. We note that $F(x,y,\theta)$ is smooth and strongly convex.

Moreover, we do not make any distributional assumptions (like standard Gaussian as in \cite{pal2022learning,ghosh_agnostic}) on  data. We only require the data points to be independent. Hence, even with practical heavy tailed data, our formulation continues to work.

We propose and analyze gradient EM algorithm for general agnostic mixtures. We assume that the initial estimates are within $\Theta(1)$ of the population loss minimizers $\{\theta^*_j\}_{j=1}^k$. With an appropriate step size, without any distributional assumption on data, we show that the gradient EM algorithm converges exponentially fast in the neighborhood of $\{\theta^*_j\}_{j=1}^k$.

\subsection{Setup}
Recall that the parameters $\theta_1^*, \ldots, \theta_k^*$ minimize the population loss function. For each $j \in [k]$, define
\[
S_j^* \;=\; \Bigl\{(x \in \mathbb{R}^d,\, y \in \mathbb{R}) \;:\; F(x,y;\theta^*_j) 
< F(x,y;\theta^*_l) \Bigr\} ,
\]
for all $ l \in [k]\setminus\{j\}$, as the set of observations for which $\theta_j^*$ provides a strictly better predictor (in terms of the loss function $F$) than all other parameters $\theta_1^*, \ldots, \theta_k^*$.

To rule out degenerate cases, we assume that , we have
\begin{align*}
    \Pr_{\mathcal D}\bigl((x,y):(x,y) \in S_j^*\bigr) \;\ge\; \pi_{\min}, \quad \text{for} \,\, j \in [k]
\end{align*}
for some constant $\pi_{\min} > 0$. Here, we focus on the joint probability measure induced by the pair $(x,y)$. Note that in the generative setup, our definitions of $S_j^*$ and $\pi_{\min}$ are analogous to those used in \cite{yi2014alternating, yi2016solving}. 

Throughout the paper, we work with the base loss functions which satisfy the following assumption.

\begin{assumption}[Strong Convexity and Smoothness]
    The loss function $F(.,.,\theta)$ is $M$-smooth and $m$-strongly convex with respect to $\theta$ almost surely.
    \label{asm:structure}
\end{assumption}
In the above section, we have discussed several problem instances where the above assumption holds. 

Moreover, since our goal is to recover $\{\theta^*_j\}_{j=1}^k$ and there are no ground truth parameters, we need a few geometric quantities. We define the \emph{misspecification} parameters $(\epsilon,\epsilon_1)$ in the following way. For all $j \in [k]$ and $(x,y) \in S^*_j$, almost surely we have 
\begin{align}
    F(x,y;\theta^*_j) \leq \epsilon, \quad \|\nabla F(x,y;\theta^*_j)\| \leq \epsilon_1.
    \label{eqn:misspec}
\end{align}
The above implies that the base loss and its gradients are small at the population loss minimizers. For mixed linear regression in the generative setup, $\epsilon = \epsilon_1 = 0$. Such misspecification conditions were also considered prior in literature in the agnostic setting (see \cite{pal2022learning,ghosh_agnostic}).

Furthermore, we also require a \emph{separation} condition for the problem to be identifiable. We assume that for all $(x,y) \in S^*_l$, almost surely
\begin{align}
    F(x,y;\theta^*_j) > \Delta,  \quad l \neq j \quad l\in [k] \,\,j\in [k]
    \label{eqn:sep}
\end{align}
for some $\Delta > 0$. In the generative setup, the separation assumption is usually given in terms of the ground truth parameters (see \cite{yi2014alternating,yi2016solving}). However in the agnostic setting as shown in \cite{pal2022learning,ghosh_agnostic}, the separation is given in terms of the \emph{base} loss functions. 

\subsection{Summary of Contributions}
We now describe the main results of the paper. We state the results informally here whereas rigorous statement can be found in Section~\ref{subsec:guarantee}. 

Our main contribution is the theoretical analysis for the gradient EM algorithm in the agnostic setup with general mapping function and general (strongly convex and smooth) base losses. In a nutshell, at the $t$-th iteration, the gradient EM algorithm uses the current estimate of $\{\theta^*_j\}_{j=1}^k$, given by $\{\theta^{(t)}_j\}_{j=1}^k$ for computing the \emph{soft-min} probabilities $p_{\theta^{(t)}_1,\ldots,\theta^{(t)}_k}(x_i,y_i;\theta^{(t)}_j)$ for all $j \in [k]$ and $i \in [n]$. After this, the algorithm takes the gradient of the \emph{soft-min} loss function and takes a step with learning rate (step size) $\gamma$.

We show that the iterates of gradient EM algorithm after $T$ iterations given by $\{\theta_j^{(T)}\}_{j=1}^k$ satisfy
\begin{align*}
    \|\theta_j^{(T)} - \theta^*_j\| \leq r^T \| \theta_j^{(0)} - \theta^*_j\| + \zeta,
\end{align*}
with high probability, where $r<1$ provided $$ \| \theta_j^{(0)} - \theta^*_j\| \leq c_{ini}\|\theta^*_j\|.$$

Here $c_{ini}$ is denoted as initialization parameter and $\zeta$ is the (small) error floor, implying that the iterates converge to a neighborhood of $\{\theta^*_j\}_{j=1}^k$ dependent on $\zeta$. Note that the error floor comes from the gradient EM update as well as the agnostic setting. In \cite{balakrishnan2017statistical}, it is shown that such an error floor is unavoidable even for the mixed linear regression problem in the generative setup. Furthermore, in the agnostic setup, \cite{pal2022learning,ghosh_agnostic} also obtain similar error floor. As far as the initialization parameter is concerned, we show that $c_{ini} = \Theta(1)$ is sufficient which is an improvement over \cite{pal2022learning}. 

One issue that arises in the agnostic setup is that the minimizers of $F(.)$ can be quite different from $\{\theta^*_j\}_{j=1}^k$. In the generative setup, these are the same. Hence the analysis become non trivial since we need to understand the behavior of $F(.)$ in the neighborhood of $\{\theta^*_j\}_{j=1}^k$. We use the structure of $F(.)$ (smoothness and strong convexity) along with the misspecification and separation condition to achieve this. 

There are many technical challenges in proving contraction for gradient EM, specially in the agnostic setting for general mixtures. In the mixed linear regression, even in the agnostic setting, we have an exact expression of \emph{base} loss and its gradient. With assumptions like Gaussian on the data, one can compute the distribution and behavior of \emph{base} loss and its gradient in the mixed linear regression setup. Previous works like \cite{yi2016solving,pal2022learning,ghosh_agnostic} leverage this crucially to obtain the convergence of gradient EM. In our case, we need to rely on the behavior of $F(.)$ only.

We first show that at the $t$-th iteration of gradient EM, for $(x_i,y_i) \in S^*_j$, the soft-min probability $$p_{\theta^*_1,\ldots,\theta^*_k}(x_i,y_i;\theta^{(t)}_j) \geq 1- \eta$$ 
with small $\eta$. Furthermore, thanks to the separation condition, we show that for $(x_i,y_i) \notin S^*_j$, $$p_{\theta^*_1,\ldots,\theta^*_k}(x_i,y_i;\theta^{(t)}_j) \leq \eta'$$ with small $\eta'$.

In order to make the analysis of gradient EM tractable, we employ re-sampling in each iteration. In particular, this removes the inter-iteration dependency. Such sample splitting techniques, while not ideal, is ubiquitous in the analysis of EM type algorithms. For example, see \cite{yi2014alternating,yi2016solving,ghosh2020alternating} for generative mixed linear regression,  \cite{pal2022learning,ghosh_agnostic} for agnostic mixed linear regression and \cite{netrapalli2015phase} for phase retrieval. One way to remove such sample split is via Leave One Out (LOO) analysis. In \cite{chen2019gradient}, this technique is used for simpler problems like phase retrieval. In the general mixture problems in agnostic setup, using such a technique is quite non trivial.

Finally we comment that studying agnostic setup posses additional challenge in the convergence study of EM. In \cite{balakrishnan2017statistical}, the authors study the population update first and then connect it with the empirical update. However, as pointed out in \cite{ghosh_agnostic}, obtaining the population update with soft-min loss is complicated. Hence, we do not take that route and analyze the empirical update only.

\subsection{Other Related Works}
Iterative algorithms like EM or its hard variant Alternating Minimization (AM) are primarily used for problems involving latent variables. Examples for AM include phase retrieval (\cite{netrapalli2015phase}), matrix sensing and completion (\cite{jain2013low}), mixed linear regression (\cite{yi2016solving}), max-affine regression (\cite{ghosh2019max}) and clustered Federated learning \cite{ghosh2020efficient}.

On the other hand, in the seminal paper by \cite{balakrishnan2017statistical}, the analysis of EM is done for a variety of problems including Gaussian mean estimation and mixed linear regression. The above works assume a suitable initialization to ensure convergence. In \cite{kwon2018global}, it is shown that for two symmetric mixture problems, EM converges even with random initialization.

In all the above works, it is assumed that the data is Gaussian, which is a crucial requirement for proving contraction in parameter space. In \cite{ghosh2020max}, the Gaussian assumption is relaxed, but is replaced with sub-Gaussian assumption with heavy ball condition.

In a parallel line of work, \cite{shen2019iterative,ghosh2020alternating} study the convergence speed of AM/EM type algorithms and show that with suitable initialization, they converge double exponentially. Later \cite{chandrasekher2021sharp} observe the same phenomenon for a larger class of algorithms for non-convex optimization.

Note that none of these works are directly comparable with our setup. These works assume a generative model on data. Recently, in the agnostic setup \cite{pal2022learning,ghosh_agnostic} study mixed linear regression and study EM/AM algorithms. Our work can be seen as a direct follow up to these above works. The discussion and comparison with respect to \cite{pal2022learning,ghosh_agnostic} were presented earlier thoroughly.

\subsection{Notation}
We collect a few notation used throughout the paper here. Note that $\|.\|$ denote the $\ell_2$ norm unless otherwise stated. Also $\langle .,. \rangle$ denote the usual ($\ell_2$) inner product. We use $c,c_1,c_2,..,C,C_1,..$ for positive universal constants, the value of which may differ from instance to instance.

\section{Main Results: Algorithm and Theoretical Guarantees}
In this section we first present the gradient EM algorithm and then study its convergence properties.
\subsection{Algorithm}
\label{subsec:algo}
The details of the algorithm is presented in Algorithm~\ref{alg:grad_em}. We use re-sampling where a fresh set of samples is used in each iteration. Suppose we run the algorithm for $T$ iterations. We split the data points in $T$ disjoint datasets each containing $n' = \lfloor \frac{n}{T} \rfloor$.

At the first step, with current estimates $\{\theta^{(t)}_j\}_{j=1}^k$, we compute the soft-min probabilities $p_{\theta^{(t)}_1,..,\theta^{(t)}_k}(x_i^{(t)},y_i^{(t)};\theta^{(t)}_j)$ using Equation~\ref{eq:softmin}. The algorithm then takes a gradient step by weighting the gradient of the (base) loss functions with this newly computed soft-min probabilities. These two steps alternate upto $T$ iteration where the algorithm outputs the final set of estimates $\{\theta^{(T)}_j\}_{j=1}^k$.
\begin{algorithm}[H]
	\caption{Gradient EM}
	\begin{algorithmic}[1]
		\STATE  \textbf{Input:} $\{x_i,y_i\}_{i=1}^n$, Step size $\gamma$ 
		\STATE \textbf{Initialization:} Initial iterate $\{\theta^{(0)}_j\}_{j=1}^k$  \\
        \STATE Split all samples into $T$ disjoint datasets $\{x_i ^{(t)},y_i^{(t)}\}_{i=1}^{n'}$ with $n' = n/T$ for all $t =0,1,\ldots,T-1$ \\
		\FOR{$t=0,1, \ldots, T-1 $}
		\STATE \underline{Compute Probabilities:}  \\
		\STATE Compute $p_{\theta^{(t)}_1,..,\theta^{(t)}_k}(x_i^{(t)},y_i^{(t)};\theta^{(t)}_j)$ for all $j \in [k]$ and $i \in [n']$
		\STATE \underline{Gradient Step:} (for all $j \in [k]$)
		\begin{align*}
			\theta^{(t+1)}_j = \theta^{(t)}_j - \frac{\gamma}{n'}\sum_{i =1}^{n'} \Tilde{p}(\theta^{(t)}_j) \nabla F(x_i^{(t)},y_i^{(t)};\theta^{(t)}_j), 
		\end{align*}
        \begin{align*}
            \text{where} \quad \Tilde{p}(\theta^{(t)}_j) = p_{\theta^{(t)}_1,..,\theta^{(t)}_k}(x_i^{(t)},y_i^{(t)};\theta^{(t)}_j)
        \end{align*}
		\ENDFOR
		\STATE \textbf{Output:} $\{\theta^{(T)}_j\}_{j=1}^k$
	\end{algorithmic}
	\label{alg:grad_em}
\end{algorithm}

\subsection{Theoretical Guarantees}
\label{subsec:guarantee}
Here, we present the convergence guarantees for Algorithm~\ref{alg:grad_em}. We leverage the assumptions like strong convexity or smoothness on the (random) base loss $F(x,y,\theta)$. 

We now present the main result of this paper, which is a contraction in parameter using EM algorithm. We give the result for one iteration only.

\vspace{2.5mm}
\begin{theorem}[Gradient EM]
\label{thm:grad_em}
Suppose that Assumption~\ref{asm:structure} holds and we run gradient EM for one iteration with parameter estimates $\{\theta_j\}_{j=1}^k$. Let $\max_{j\in [k]} \|\theta^*_j\| \leq 1$ and furthermore, we have the initialization condition
	\begin{align*}
		\|\theta_j - \theta^*_j\| \leq c_{\mathsf{ini}} \|\theta^*_j\|,
	\end{align*}
	for all $j \in [k]$, where $c_{\mathsf{ini}}$ is a small positive constant (initialization parameter). There exists universal constants $c_1$ and $c_2$ such that one iteration of the gradient EM algorithm with step size $\gamma$ yields $\{\theta^+_j\}_{j=1}^k$ satisfying
	\begin{align*}
     \|\theta^+_j - \theta^*_j\| \leq \left (1-c \gamma \pi_{\min} m(1-\eta) \right )^{1/2}\|\theta_j-\theta^*_j\| + \zeta,
\end{align*}
with probability at least $1- c_1\exp(-c_2\pi_{\min}n')$, where
\begin{align*}
    \zeta &= \gamma \epsilon_1 +(\gamma \epsilon_1 c_{ini})^{1/2} + \gamma \eta' (2+ \epsilon_1 + Mc_{ini}).
\end{align*}
Here $\eta,\eta'$ are given by
\begin{align*}
    \eta = 1-  \frac{e^{-\beta (\epsilon + \epsilon_1 c_{ini} + \frac{M}{2} c_{ini}^2)}}{1 + (k-1) e^{-\beta (\Delta - (\epsilon_1 + 2 M) c_{ini})}}
\end{align*}
\begin{align*}
    \text{and} \quad \eta' = \frac{e^{-\beta (\Delta - (\epsilon_1 + 2M)c_{ini})}}{e^{-\beta(\epsilon + \epsilon_1 c_{ini} + \frac{M}{2}c_{ini}^2) }}.
\end{align*}
\end{theorem}
\vspace{2.5mm}
The proof of Theorem~\ref{thm:grad_em} are deferred to the Appendix of the full paper. We now collect some remarks here.

\vspace{2mm}
\begin{remark}[Contraction]
    With $\gamma$ as sufficiently small constant, we see that the term $[1-c \gamma \pi_{\min} m(1-\eta)]^{1/2}$ is $<1$, which implies a contraction. Hence, the convergence speed of gradient EM is exponential.
\end{remark}
\vspace{2mm}
\begin{remark}[Error Floor]
    Note that the one step progress comes with an error floor $\zeta$. As mentioned earlier, even in the non-agnostic setting for mixed linear regression, \cite{balakrishnan2017statistical} shows that such an error floor is unavoidable. Also, in the agnostic setup, for mixed linear regression \cite{pal2022learning,ghosh_agnostic} shows a similar error floor. Hence, we can expect an error floor for general mixtures in the agnostic setting.
\end{remark}
\vspace{2mm}
\begin{remark}[Terms in $\zeta$]
    The error floor depends on the learning rate $\gamma$, misspecification parameters $(\epsilon,\epsilon_1)$, separation $\Delta$, initialization condition $c_{ini}$ as well as the smoothness and strong convexity parameters. The error floor depends with linearly with misspecification $\epsilon_1$. In the (non-agnostic) mixed linear regression setup, $\epsilon_1 = 0$. Similar effect of model misspecification is observed in \cite{pal2022learning,ghosh_agnostic}. 
\end{remark}
\vspace{2mm}
\begin{remark}[Terms $(\eta,\eta')$]
    Note that the terms $(\eta,\eta')$ are small, provided the separation is large and misspecification is small. This ensures that the error floor is small.
\end{remark}
\vspace{2mm}
\begin{remark}[No distributional assumption]
    Note that provided the expected (base) loss is strongly convex and smooth, we do not have any additional distributional assumption on data. Also, unlike previous results in the mixed linear regression setup (\cite{yi2016solving,ghosh_agnostic}), we do not require any condition on on sample complexity $n'$.
\end{remark}
\subsection{Proof Sketch}
We provide a brief proof sketch here. Using the iterate of gradient EM, and focusing on $j=1$, we have
\begin{align*}
		& \|\theta^+_1 - \theta^*_1\| \\
        &= \| \theta_1 - \theta^*_1 -  \frac{\gamma}{n'} \sum_{i=1}^{n'} p_{\theta_1,\ldots,\theta_k}(x_i,y_i;\theta_1)\nabla F(x_i,y_i;\theta_1)\|.
	\end{align*}
We break the sum into $\{i: (x_i,y_i) \in S^*_1\}$ and $\{i: (x_i,y_i) \notin S^*_1\}$. Note that if $\{i: (x_i,y_i) \in S^*_1\}$, we show that $p_{\theta_1,\ldots,\theta_k}(x_i,y_i;\theta_1) \geq 1-\eta$ for a small enough $\eta$ with high probability. Using this and the fact that $F(.)$ is $M$ smooth and $m$ strongly convex, we get a contraction term. Here we use the initialization condition crucially the behavior of $F(.)$ close to $\theta^*_1$. As a result, analyzing $\{i: (x_i,y_i) \in S^*_1\}$ gives the necessary contraction needed for the convergence of gradient EM.

We then look at $\{i: (x_i,y_i) \notin S^*_1\}$. Here, we show that $p_{\theta_1,\ldots,\theta_k}(x_i,y_i;\theta_1) \leq \eta'$, for a small enough $\eta'$ with high probability. We use the separation and misspecification condition crucially here. This results in a small error floor. Combining these two, we obtain the results in Theorem~\ref{thm:grad_em} with a contraction term and an error floor.

\subsection{Conclusion and Open Problems}
We address the problem of mixture with general losses in agnostic setting. We propose and analyze gradient EM algorithm and show that provided separation and initialization condition it converges exponentially. We believe that EM (or gradient EM) can be used in a broader context in agnostic learning. We end the paper with a few open problems. 

Can we relax the assumptions of strong convexity and smoothness on base loss functions? Can we relax the re-sampling and use techniques like Leave One Out (LOO) for agnostic mixture problems? What are some other algorithms for studying agnostic mixtures? We keep these as our future endeavors.

\bibliographystyle{unsrt}
\bibliography{ref}

\newpage
\section{Proof of the Theorem}
We focus on the iterate $\theta^+_1$ and show that the distance with $\theta^*_1$ reduces (hence contraction) over one iteration. We have
\begin{align*}
		&\|\theta^+_1 - \theta^*_1\| \\
        &= \| \theta_1 - \theta^*_1 -  \frac{\gamma}{n'} \sum_{i=1}^{n'} p_{\theta_1,\ldots,\theta_k}(x_i,y_i;\theta_1)\nabla F(x_i,y_i;\theta_1)\|
	\end{align*}
    
\emph{Shorthand:} We use the shorthand $p(\theta_1)$ to denote $p_{\theta_1,\ldots,\theta_k}(x_i,y_i;\theta_1)$, $p(\theta^*_1)$ for $p_{\theta^*_1,\ldots,\theta^*_k}(x_i,y_i;\theta^*_1)$ and $F_i(\theta_1)$ for $F(x_i,y_i;\theta_1)$ respectively. Also, we use $\{i \in S^*_1\}$ to denote $\{i:(x_i,y_i)\in S^*_1\}$. Similarly we use $|S^*_j|$ to denote $|\{i:(x_i,y_i)\in S^*_j\}|$ for $j\in [k]$. 

\vspace{2mm}
With this, we have
\begin{align*}
    &\|\theta^+_1 - \theta^*_1\|= \| \theta_1 - \theta^*_1 -  \frac{\gamma}{n'} \sum_{i=1}^{n'} p(\theta_1)\nabla F_i(\theta_1)\| \\
    & = \underbrace{\|\theta_1 - \theta^*_1 -  \frac{\gamma}{n'} \sum_{i\in S^*_1} p(\theta_1)\nabla F_i(\theta_1)\|}_{T_1} \\
    & \quad + \underbrace{\frac{\gamma}{n'}\| \sum_{i \notin S^*_1}p(\theta_1)\nabla F_i(\theta_1)\|}_{T_2}
\end{align*}
Let us look at $T_1$ first. We have
We have the following
\begin{align*}
    T_1 &= \|\theta_1 - \theta^*_1 -  \frac{\gamma}{n} \sum_{i\in S^*_1} p(\theta_1)\nabla F_i(\theta_1)\|
\end{align*}
Let $\hat{\gamma} = \gamma \frac{|S^*_1|}{n'}$ and condition on $S^*_1$. We have
\begin{align*}
    T_1^2 &= \|\theta_1 - \theta^*_1 - \hat{\gamma} \frac{1}{|S_1^*|}\sum_{i \in S^*_1}p(\theta_1) \nabla F_i(\theta_1)\|^2 \\
    & = \|\theta_1 -\theta^*_1\|^2 -\frac{2\hat{\gamma}}{|S^*_1|}\sum_{i \in S^*_1} p(\theta_1) \langle  \nabla F_i(\theta_1), \theta_1 - \theta^*_1 \rangle \\
    & + \hat{\gamma}^2\|\frac{1}{|S^*_1|}\sum_{i \in S^*_1}p(\theta_1)\nabla F_i(\theta_1)\|^2.
\end{align*}
Recall that $F_i(.)$ is $m$-strongly convex. Using that, we obtain
\begin{align*}
    &\langle  \nabla F_i(\theta_1), \theta_1 - \theta^*_1 \rangle \\
    & \geq \langle  \nabla F_i(\theta^*_1), \theta_1 - \theta^*_1 \rangle + \frac{m}{2}\|\theta_1 - \theta^*_1\|^2 \\
    & \geq \frac{m}{2}\|\theta_1 - \theta^*_1\|^2 - \|\nabla F_i(\theta^*_1)\| \| \theta_1 - \theta^*_1\| \\
    & \geq \frac{m}{2}\|\theta_1 - \theta^*_1\|^2 - \epsilon_1 c_{ini} \|\theta^*_1\|
    \end{align*}
where we use the Cauchy Schwartz inequality. For the third term in square, we have
\begin{align*}
    \| \nabla F_i(\theta_1)\| &\leq  \epsilon_1,
\end{align*}
where we have used the misspecification condition. Using $|p(\theta_1)| \leq 1$, we have
\begin{align*}
    \hat{\gamma}^2\|\frac{1}{|S^*_1|}\sum_{i \in S^*_1}p(\theta_1)\nabla F_i(\theta_1)\|^2 \leq \hat{\gamma}^2 \epsilon_1^2
\end{align*}
Substituting the above, we obtain
\begin{align*}
    T_1^2 &\leq \|\theta_1 - \theta^*_1\|^2 + \hat{\gamma}^2 \epsilon_1^2 \\
    & - \hat{\gamma}m (1-\eta) \|\theta_1 - \theta^*_1\|^2 + \hat{\gamma}\epsilon_1 c_{ini} \|\theta^*_1\| \\
    & \leq [1-\hat{\gamma}m(1-\eta)]\|\theta_1 - \theta^*_1\|^2   \\
    & + \hat{\gamma}^2 \epsilon_1^2 + \hat{\gamma} \epsilon_1 c_{ini} \|\theta^*\|\\
    & \leq [1-\hat{\gamma}m(1-\eta)]\|\theta_1 - \theta^*_1\|^2   \\
    & + \hat{\gamma}^2 \epsilon_1^2 + \hat{\gamma} \epsilon_1 c_{ini},
\end{align*}
where we use Lemma~\ref{lem:p_large} (formally given in Section~\ref{subsec:aux_lem}) to lower bound $p(\theta_1)$ uniformly. We also use the fact that $\max_{j \in [k]}\|\theta^*_j\| \leq 1$.

Hence,
\begin{align*}
    T_1 &\leq [1-\hat{\gamma}m(1-\eta)]^{1/2}\|\theta_1-\theta^*_1\| \\
    &+ \hat{\gamma}\epsilon_1 +(\hat{\gamma}\epsilon_1 c_{ini})^{1/2}
\end{align*} 

Finally, we look at $T_2$. We have
\begin{align*}
    T_2 &= \frac{\gamma}{n'}\| \sum_{i \notin S^*_1}p(\theta_1)\nabla F_i(\theta_1)\| \\
    & = \frac{\gamma}{n'} \| \sum_{j=2}^k \sum_{i \in S^*_j}p(\theta_1)\nabla F_i(\theta_1)\| \\
    & \leq \frac{\gamma \eta'}{n'} \| \sum_{j=2}^k \sum_{i \in S^*_j}\nabla F_i(\theta_1)\| \\
    &\leq \frac{\gamma \eta'}{n'} \sum_{j=2}^k \|  \sum_{i \in S^*_j}\nabla F_i(\theta_1)\|
\end{align*}
Using Lemma~\ref{lem:p_small}, we know that for $i\notin S^*_1$, we have $p(\theta_1) \leq \eta_1$.
Using, $\hat{\gamma}_j = \gamma \frac{|S^*_j|}{n'}$, we can write
\begin{align*}
    T_2 &\leq  \frac{\hat{\gamma} \eta'}{n'} \sum_{j=2}^k \| \sum_{i \in S^*_j} \nabla F_j(\theta_1) \| \\
    & \leq  \frac{\hat{\gamma} \eta'}{n'} \sum_{j=2}^k \sum_{i \in S^*_j} \| \nabla F_j(\theta_1) \|.
\end{align*}
We calculate the following:
\begin{align*}
    \| \nabla F_j(\theta_1) \|&= \|\nabla F_j(\theta^*_j) + \nabla F_j(\theta_1) - \nabla F_j(\theta^*_j)\| \\
    &\leq \|\nabla F_j(\theta^*_j) \| + \|\nabla F_j(\theta_1) - \nabla F_j(\theta^*_j)\| \\
    & \leq  \epsilon_1 + M \|\theta_1 - \theta^*_j\| \\
    & \leq \epsilon_1 + M \|\theta_1 - \theta^*_1 + \theta^*_1 - \theta^*_j\| \\
    & \leq \epsilon_1 + M c_{ini} \|\theta^*_1\| + \|\theta^*_1 - \theta^*_j\| \\
    & \leq \epsilon_1 + M c_{ini} +2.
\end{align*}
using the initialization condition and the fact that $\max_{j\in [k]} \|\theta^*_j\| \leq 1$.

We finally obtain bounds on $|S^*_j|$ for $j\in [k]$. Note that trivially $|S^*_j| \leq n'$ almost surely since there are $n'$ observations. Also, recall that we assume that for every $j \in [k]$,
\[
\Pr_{\mathcal D}\bigl((x,y) \in S_j^*\bigr) \;\ge\; \pi_{\min},
\]
for some constant $\pi_{\min} > 0$. Hence $\E |S^*_j| > \pi_{\min} n'$. Moreover, using the standard binomial concentration to show that $|S^*_j| = |i:(x_i,y_i)\in S^*_j| \gtrsim \pi_{\min} n'$ with probability at least\footnote{Write $|S^*_j|$ as sum of indicators.} $1-\exp(-c_1\pi_{\min}n')$.

\emph{Completing the proof:} We now put in the above bounds and conclude the proof. We have
\begin{align*}
    \|\theta^+ - \theta^*_1\| \leq T_1 + T_2,
\end{align*}
where
\begin{align*}
    T_1 &\leq[1-\hat{\gamma}m(1-\eta)]^{1/2}\|\theta_1-\theta^*_1\| \\
    &+ \hat{\gamma}(\epsilon_1) +(\hat{\gamma}\epsilon_1 c_{ini})^{1/2} 
\end{align*}
We use the lower and upper bounds on $|S^*_1|$ now. We obtain
\begin{align*}
    T_1 &\leq[1-c \gamma \pi_{\min} m(1-\eta)]^{1/2}\|\theta_1-\theta^*_1\| \\
    &+ \gamma (\epsilon_1) +(\gamma \epsilon_1 c_{ini})^{1/2} 
\end{align*}
with probability at least $1- \exp(-c_2\pi_{\min}n')$. 

Let us look at $T_2$. 
\begin{align*}
    T_2 &\leq  \frac{\gamma \eta'}{n'} \sum_{j=2}^k \sum_{i \in S^*_j} (2+\epsilon_1+Mc_{ini})\\
    & \leq  \gamma \eta' (2+ \epsilon_1 + Mc_{ini}),
\end{align*}
with probability at least $1- \exp(-c_2 \pi_{\min}n')$. 

Combining the above, finally we obtain
\begin{align*}
     \|\theta^+ - \theta^*_1\| \leq[1-c \gamma \pi_{\min} m(1-\eta)]^{1/2}\|\theta_1-\theta^*_1\| + \zeta,
\end{align*}
with probability at least $1 - c_1\exp(-c_2 \pi_{\min}n')$, where
\begin{align*}
    \zeta &= \gamma \epsilon_1 +(\gamma \epsilon_1 c_{ini})^{1/2} + \gamma \eta' (2+ \epsilon_1 + Mc_{ini}),
\end{align*}
which concludes the proof.

\vspace{5mm}
\subsection{Auxiliary Lemmas}
\label{subsec:aux_lem}
We now prove the auxiliary lemmas required for the proof of the main theorem. 

\begin{lemma}
\label{lem:p_large}
    For any $(x_i,y_i) \in S^*_j$, we have 
    \begin{align*}
        p_{\theta_1,\ldots,\theta_k}(x_i,y_i;\theta_j) \geq 1-\eta
    \end{align*}
   almost surely, where
    \begin{align*}
    \eta = 1- \frac{e^{-\beta (\epsilon + \epsilon_1 c_{ini} + \frac{M}{2} c_{ini}^2)}}{1 + \sum_{l \neq j} e^{-\beta (\Delta - (\epsilon_1 + 2 M) c_{ini})}}.
\end{align*}
\end{lemma}

\emph{Proof of Lemma:} For any $(x_i,y_i) \in S^*_j$, we write 
\begin{align*}
    p_{\theta_1,..,\theta_k}(x,y;\theta_j) = \frac{e^{-\beta F(x,y;\theta_j)}}{\sum_{l=1}^k e^{-\beta F(x,y;\theta_l)}}.
\end{align*}
Let us first compute an upper bound on $F(x,y;\theta_j)$. We continue using the shorthand $F_j(\theta_j)$ for this. We have
\begin{align*}
    F_j(\theta_j) &= F_j(\theta^*_j) +  F_j(\theta_j) - F_j(\theta^*_j) \\
    & \leq F_j(\theta^*_j) + \langle \nabla F_j(\theta^*_j),\theta_j - \theta^*_j \rangle + \frac{M}{2}\|\theta_j - \theta^*_j\|^2 \\
    & \leq \epsilon + \|\nabla F_j(\theta^*_j)\| \|\theta_j - \theta^*_j\| + \frac{M}{2}\|\theta_j - \theta^*_j\|^2 \\
    & \leq \epsilon + \epsilon_1 c_{ini} + \frac{M}{2} c_{ini}^2.
\end{align*}
Here we use the smoothness of $F_j(.)$ in the second line. Moreover we use the initialization condition and assumptions on $F_j(\theta^*_j)$ and the gradient of it.

We now need to prove a lower bound on $F(x,y;\theta_l)$ for $l \neq j$. We have
\begin{align*}
    F_j(\theta_l) &= F_j(\theta^*_l) + F_j(\theta_l) - F_j(\theta^*_l) \\
    & \geq \Delta + \langle \nabla F_j(\theta^*_l),\theta_l - \theta^*_l \rangle,
\end{align*}
where we use the separation condition and the convexity of $F_j(.)$. Continuing, we obtain
\begin{align*}
    F_j(\theta_l) &\geq \Delta - \|\nabla F_j(\theta^*_l)\| \|\theta_l - \theta^*_l\| \\
    & \geq \Delta - (\|\nabla F_j(\theta^*_j)\| + \|\nabla F_j(\theta^*_l) - \nabla F_j(\theta^*_j) \|) c_{ini} \\
    & \geq \Delta - (\epsilon_1 + M \|\theta^*_l - \theta^*_j\|)c_{ini} \\
    & \geq \Delta - (\epsilon_1 + 2 M) c_{ini}.
\end{align*}
Here, we once again use the smoothness of the base loss along with the initialization condition. Substituting, we obtain
\begin{align*}
    & p_{\theta_1,..,\theta_k}(x,y;\theta_j) = \frac{e^{-\beta F(x,y;\theta_j)}}{\sum_{l=1}^k e^{-\beta F(x,y;\theta_l)}} \\
    & \geq  \frac{e^{-\beta (\epsilon + \epsilon_1 c_{ini} + \frac{M}{2} c_{ini}^2)}}{1 + \sum_{l \neq j} e^{-\beta F(x,y;\theta_l)}} \\
    & \geq  \frac{e^{-\beta (\epsilon + \epsilon_1 c_{ini} + \frac{M}{2} c_{ini}^2)}}{1 + \sum_{l \neq j} e^{-\beta (\Delta - (\epsilon_1 + 2 M) c_{ini})}} = 1-\eta,
\end{align*}
where
\begin{align*}
    \eta = 1 - \frac{e^{-\beta (\epsilon + \epsilon_1 c_{ini} + \frac{M}{2} c_{ini}^2)}}{1 + \sum_{l \neq j} e^{-\beta (\Delta - (\epsilon_1 + 2 M) c_{ini})}}
\end{align*}
The above event occurs with probability one.

\vspace{5mm}
\begin{lemma}
\label{lem:p_small}
    Suppose $(x_i,y_i) \notin S^*_j$. We have
    \begin{align*}
        p_{\theta_1,\ldots,\theta_k}(x_i,y_i;\theta_j) \leq \eta',
    \end{align*}
almost surely where 
\begin{align*}
    \eta' = \frac{e^{-\beta (\Delta - (\epsilon_1 + 2M)c_{ini})}}{e^{-\beta(\epsilon + \epsilon_1 c_{ini} + \frac{M}{2}c_{ini}^2) }}.
\end{align*}
\end{lemma}

\emph{Proof of Lemma:} Since $\{S^*_j\}_{j=1}^k$ partitions $\mathbb{R}^d$, we have $(x_i,y_i) \in S^*_{j'}$ for some $j' \neq j$. We have
\begin{align*}
    p_{\theta_1,\ldots,\theta_k}(x_i,y_i;\theta_j)  =\frac{e^{-\beta F(x,y;\theta_j)}}{\sum_{l=1}^k e^{-\beta F(x,y;\theta_l)}}. 
\end{align*}
As derived in Lemma~\ref{lem:p_large}, for $j' \neq j$, we have 
\begin{align*}
     F(x,y;\theta_j) \geq \Delta - (\epsilon_1 + 2M)c_{ini}.
\end{align*}
Continuing, we have
\begin{align*}
    p_{\theta_1,\ldots,\theta_k}(x_i,y_i;\theta_j) &\leq \frac{e^{-\beta (\Delta - (\epsilon_1 + 2M)c_{ini})} }{e^{-\beta F(x,y;\theta_{j'})}+\sum_{l\neq j'} e^{-\beta F(x,y;\theta_l)}} \\
    &\leq \frac{e^{-\beta (\Delta - (\epsilon_1 + 2M)c_{ini})} }{e^{-\beta F(x,y;\theta_{j'})}+0}
\end{align*}
Using the upper-bound from Lemma~\ref{lem:p_large}, we have
\begin{align*}
    F(x,y;\theta_{j'}) \leq \epsilon + \epsilon_1 c_{ini} + \frac{M}{2}c_{ini}^2,
\end{align*}
almost surely. Hence, we obtain
\begin{align*}
    p_{\theta_1,\ldots,\theta_k}(x_i,y_i;\theta_j) \leq \frac{e^{-\beta (\Delta - (\epsilon_1 + 2M)c_{ini})}}{e^{-\beta(\epsilon + \epsilon_1 c_{ini} + \frac{M}{2}c_{ini}^2) }}.
\end{align*}
Hence,
\begin{align*}
    \eta' = \frac{e^{-\beta (\Delta - (\epsilon_1 + 2M)c_{ini} )}}{e^{-\beta(\epsilon + \epsilon_1 c_{ini} + \frac{M}{2}c_{ini}^2) }}.
\end{align*}
\end{document}